\documentclass{article}

\usepackage{times}
\usepackage{graphicx} 
\usepackage{subcaption}
\graphicspath{ {img/} }

\usepackage{natbib}

\usepackage{algorithm}
\usepackage{algorithmic}

\usepackage{amsmath}
\usepackage{amssymb}
\usepackage{amsthm}
\usepackage{thmtools}

\usepackage{hyperref}

\usepackage[accepted]{icml2017}

\icmltitlerunning{Constrained Policy Optimization}

\newcommand{\tauE}[2]{\underset{\tau \sim #1}{\E}\left[ #2 \right]}

\newtheorem*{theorem*}{Theorem}
\newcommand{\smallsurr}[3]{\underset{\begin{subarray}{c} s \sim d^{#1} \\ a \sim #2 \end{subarray}}{\E} \left[ #3 \right]}

\newcommand{\surrA}[3]{\frac{1}{1-\gamma}\underset{\begin{subarray}{c} s \sim d^{#1} \\ a \sim #2 \end{subarray}}{\E} \left[ #3 \right]}

\begin{document} 

\twocolumn[
\icmltitle{Constrained Policy Optimization}



\icmlsetsymbol{equal}{*}

\begin{icmlauthorlist}
\icmlauthor{Joshua Achiam}{ucb}
\icmlauthor{David Held}{ucb}
\icmlauthor{Aviv Tamar}{ucb}
\icmlauthor{Pieter Abbeel}{ucb,o}
\end{icmlauthorlist}

\icmlaffiliation{ucb}{UC Berkeley}
\icmlaffiliation{o}{OpenAI}

\icmlcorrespondingauthor{Joshua Achiam}{jachiam@berkeley.edu}

\icmlkeywords{boring formatting information, machine learning, ICML}

\vskip 0.3in
]



\printAffiliationsAndNotice{}  

%
%



\newcommand{\avet}{{\mathbf  a}}
\newcommand{\bvet}{{\mathbf  b}}
\newcommand{\cvet}{{\mathbf  c}}
\newcommand{\dvet}{{\mathbf  d}}
\newcommand{\evet}{{\mathbf  e}}
\newcommand{\fvet}{{\mathbf  f}}
\newcommand{\gvet}{{\mathbf  g}}
\newcommand{\hvet}{{\mathbf  h}}
\newcommand{\ivet}{{\mathbf  i}}
\newcommand{\jvet}{{\mathbf  j}}
\newcommand{\kvet}{{\mathbf  k}}
\newcommand{\lvet}{{\mathbf  l}}
\newcommand{\mvet}{{\mathbf  m}}
\newcommand{\nvet}{{\mathbf  n}}
\newcommand{\ovet}{{\mathbf  o}}
\newcommand{\pvet}{{\mathbf  p}}
\newcommand{\qvet}{{\mathbf  q}}
\newcommand{\rvet}{{\mathbf  r}}
\newcommand{\svet}{{\mathbf  s}}
\newcommand{\tvet}{{\mathbf  t}}
\newcommand{\uvet}{{\mathbf  u}}
\newcommand{\vvet}{{\mathbf  v}}
\newcommand{\xvet}{{\mathbf  x}}
\newcommand{\yvet}{{\mathbf  y}}
\newcommand{\zvet}{{\mathbf  z}}
\newcommand{\wvet}{{\mathbf  w}}

\newcommand{\Avet}{{\mathbf  A}}
\newcommand{\Bvet}{{\mathbf  B}}
\newcommand{\Cvet}{{\mathbf  C}}
\newcommand{\Dvet}{{\mathbf  D}}
\newcommand{\Evet}{{\mathbf  E}}
\newcommand{\Fvet}{{\mathbf  F}}
\newcommand{\Gvet}{{\mathbf  G}}
\newcommand{\Hvet}{{\mathbf  H}}
\newcommand{\Ivet}{{\mathbf  I}}
\newcommand{\Jvet}{{\mathbf  J}}
\newcommand{\Kvet}{{\mathbf  K}}
\newcommand{\Lvet}{{\mathbf  L}}
\newcommand{\Mvet}{{\mathbf  M}}
\newcommand{\Nvet}{{\mathbf  N}}
\newcommand{\Ovet}{{\mathbf  O}}
\newcommand{\Pvet}{{\mathbf  P}}
\newcommand{\Qvet}{{\mathbf  Q}}
\newcommand{\Rvet}{{\mathbf  R}}
\newcommand{\Svet}{{\mathbf  S}}
\newcommand{\Tvet}{{\mathbf  T}}
\newcommand{\Uvet}{{\mathbf  U}}
\newcommand{\Xvet}{{\mathbf  X}}
\newcommand{\Yvet}{{\mathbf  Y}}
\newcommand{\Vvet}{{\mathbf  V}}
\newcommand{\Wvet}{{\mathbf  W}}
\newcommand{\Zvet}{{\mathbf  Z}}

\newcommand{\Deltavet}{\mathbf  \Delta}
\newcommand{\Lambdavet}{{\mathbf  \Lambda}}
\newcommand{\Sigmavet}{\mathbf  \Sigma}
\newcommand{\Thetavet}{{\mathbf  \Theta}}

\newcommand{\s}{ {\sigma} }

\newcommand{\e}{{\mathrm e}}
\newcommand{\jm}{{\mathrm j}}
\newcommand{\E}{{\mathrm E}}
\newcommand{\Ex}{{\mathbb E}}
\renewcommand{\d}{{\mathrm d}}
\newcommand{\dt}{{\mathrm d}t}
\newcommand{\X}{ {\mathcal X} }
\newcommand{\Y}{ {\mathcal Y} }
\newcommand{\Z}{ {\mathcal Z} }

\newcommand{\calA}{{\mathcal A}}
\newcommand{\calB}{{\mathcal B}}
\newcommand{\calC}{{\mathcal C}}
\newcommand{\calD}{{\mathcal D}}
\newcommand{\calE}{{\mathcal E}}
\newcommand{\calF}{{\mathcal F}}
\newcommand{\calG}{{\mathcal G}}
\newcommand{\calH}{{\mathcal H}}
\newcommand{\calI}{{\mathcal I}}
\newcommand{\calJ}{{\mathcal J}}
\newcommand{\calK}{{\mathcal K}}
\newcommand{\calL}{{\mathcal L}}
\newcommand{\calM}{{\mathcal M}}
\newcommand{\calN}{{\mathcal N}}
\newcommand{\calO}{{\mathcal O}}
\newcommand{\calP}{{\mathcal P}}
\newcommand{\calQ}{{\mathcal Q}}
\newcommand{\calR}{{\mathcal R}}
\newcommand{\calS}{{\mathcal S}}
\newcommand{\calT}{{\mathcal T}}
\newcommand{\calU}{{\mathcal U}}
\newcommand{\calV}{{\mathcal V}}
\newcommand{\calX}{{\mathcal X}}
\newcommand{\calY}{{\mathcal Y}}
\newcommand{\calW}{{\mathcal W}}
\newcommand{\calZ}{{\mathcal Z}}
\newcommand{\qtil}{{\tilde{q}}}
\newcommand{\td}{{\tilde{\delta}}}

\newcommand{\vect}[1]{ {\mbox{\rm vec}(#1)} }


\newcommand{\Atil}{\tilde{A}}
\newcommand{\Zhat}{\hat{Z}}
\newcommand{\Hbar}{\bar{H}}
\newcommand{\Dhat}{\hat{D}}
\newcommand{\dhat}{\hat{d}}

\newcommand{\rhat}{\hat{r}}
\newcommand{\xhat}{\hat{x}}
\newcommand{\yhat}{\hat{y}}
\newcommand{\zhat}{\hat{z}}
\newcommand{\xbar}{\bar{x}}
\newcommand{\ubar}{\bar{u}}
\newcommand{\ybar}{\bar{y}}
\newcommand{\zbar}{\bar{z}}
\newcommand{\pdot}{\dot{p}}
\newcommand{\pddot}{\ddot{p}}
\newcommand{\pbar}{\bar{p}}
\newcommand{\qdot}{\dot{q}}
\newcommand{\qddot}{\ddot{q}}
\newcommand{\qbar}{\bar{q}}
\newcommand{\xdot}{\dot{x}}
\newcommand{\ydot}{\dot{y}}
\newcommand{\zdot}{\dot{z}}
\newcommand{\yddot}{\ddot{y}}
\newcommand{\thdot}{\dot{\theta}}
\newcommand{\thddot}{\ddot{\theta}}
\newcommand{\util}{{\tilde{u}}}
\newcommand{\xtil}{{\tilde{x}}}
\newcommand{\ytil}{{\tilde{y}}}
\newcommand{\lam}{\lambda}
\newcommand{\lamax}{\lambda\ped{max}}
\newcommand{\lamin}{\lambda\ped{min}}
\newcommand{\adj}{ {\mbox{\rm adj}\;} }
\newcommand{\sign}{\mbox {\rm sgn}}
\newcommand{\spn}{\mbox {\rm span}}
\newcommand{\barJ}{\bar{J}}
\newcommand{\dom}{\mathop {\mathrm {dom}}}
\newcommand{\card}{\mathop{\mathrm{card}}}
\newcommand{\subt}{\mathop{\mathrm{s.t.}}}

\newcommand{\epi}{\mathop{\mathrm{epi}}}
\newcommand{\env}{\mathop{\mathrm{env}}}
\newcommand{\chull}{\mathop{\mathrm{co}}}
\newcommand{\graph}{\mathop{\mathrm{graph}}}
\newcommand{\prox}[1]{\mathop{\mathrm{prox}_{#1}}}
\newcommand{\sthr}[1]{\mathop{\mathrm{sthr}_{#1}}}

\def\hardsection{$\spadesuit\;$}

\newcommand{\Real}[1]{ { {\mathbb R}^{#1} } }
\newcommand{\Realp}[1]{ { {\mathbb R}_{+}^{#1} } }
\newcommand{\Realpp}[1]{ { {\mathbb R}_{++}^{#1} } }
\newcommand{\Complex}[1]{ { {\mathbb C}^{#1} } }
\newcommand{\Imag}[1]{ { {\mathbb I}^{#1} } }
\newcommand{\Field}[1]{ {\mathbb F}^{#1} }
\newcommand{\F}{ {\mathbb F}}
\newcommand{\Orth}[1]{ { {\calG_{\calO}^{#1}} } }
\newcommand{\Unit}[1]{ { {\calG_{\calU}^{#1}} } }
\newcommand{\Sym}[1]{ { {\mathbb S}^{#1} } }
\newcommand{\Symp}[1]{ { {\mathbb S}_{+}^{#1} } }
\newcommand{\Sympp}[1]{ { {\mathbb S}_{++}^{#1} } }
\newcommand{\Herm}[1]{ { {\mathbb H}^{#1} } }
\newcommand{\Skew}[1]{ { {\mathbb S\mathbb K}^{#1} } }
\newcommand{\Skherm}[1]{ { {\mathbb H\mathbb K}^{#1} } }
\newcommand{\Rman}[1]{ { {\mathcal R}^{#1} } } 
\newcommand{\Cman}[1]{ { {\mathcal C}^{#1} } }
\newcommand{\Hinf}[1]{ {  {\mathcal H}_\infty^{#1} } }
\newcommand{\RHinf}[1]{ { {\mathcal RH}_\infty^{#1} } }
\newcommand{\Htwo}[1]{ {  {\mathcal H}_2^{#1} } }
\newcommand{\RHtwo}[1]{ { {\mathcal RH}_2^{#1} } }

\newcommand{\dist}[1]{{\mathrm{dist}}{\left( #1 \right)}}
\newcommand{\diff}[2]{ \frac{\d {#1}}{\d {#2}}  }
\newcommand{\diffp}[2]{ \frac{\partial {#1}}{\partial {#2}}  }
\newcommand{\diffqd}[2]{ \frac{\d^2 {#1}}{\d {#2}^2}  }
\newcommand{\diffq}[2]{ \frac{\d^2 {#1}}{\d {#2}}  }
\newcommand{\diffqq}[3]{ \frac{\d^2 {#1}}{ \d {#2} \d {#3}  }}
\newcommand{\diffpq}[2]{ \frac{\partial^2 {#1}}{\partial {#2}^2}  }
\newcommand{\difftq}[3]{ \frac{\partial^2 {#1}}{\partial {#2}\partial {#3}}  }
\newcommand{\diffi}[3]{ \frac{\d^{#3} {#1}}{\d {#2}^{#3}}  }
\newcommand{\diffpi}[3]{ \frac{\partial^{#3} {#1}}{\partial {#2}^{#3}}  }
\newcommand{\binomial}[2]{\scriptsize{\left(\!\! \ba{c} #1 \\ #2 \ea \!\! \right)} }
\newcommand{\comb}[2]{{\left(\!\!\! \ba{c} #1 \\ #2 \ea \!\!\! \right)} }

\newcommand{\simax}{{\sigma_{\mathrm{max}}}}
\newcommand{\simin}{{\sigma_{\mathrm{min}}}}
\newcommand{\prob}{{\mbox{\rm Prob}}}
\newcommand{\var}{{\mbox{\rm var}}}
\newcommand{\sint}{{\mbox{\rm int}\,}} 
\newcommand{\relint}{{\mbox{\rm relint}\,}} 
\newcommand{\ns}{{\mbox{\tt ns}}}

\newcommand{\rank}{\mathop{\mathrm{rank}}\nolimits}
\newcommand{\range}{\mathop{\mathcal{R}}\nolimits}
\newcommand{\nulsp}{\mathop{\mathcal{N}}\nolimits}
\newcommand{\diagop}{\mathop{\mathrm{diag}}\nolimits}
\newcommand{\Var}{\mathop{\mathrm{var}}\nolimits}
\newcommand{\tr}{\mathop{\mathrm{trace}}\nolimits}
\newcommand{\sinc}{\mathop{\mathrm{sinc}}\nolimits}

\newcommand{\pre}[1]{ { {\mathop{\mathrm{Re}}}  \left({#1}\right)} }
\newcommand{\pim}[1]{ { {\mathop{\mathrm{Im}}}  ({#1})} }
\newcommand{\rp}{ ^{\Real{}} }
\newcommand{\ip}{ ^{\Imag{}} }

\newcommand{\one}{{\mathbf  1}}
\newcommand{\dss}{\displaystyle}
\newcommand{\inv}{^{-1}}
\newcommand{\pinv}{^{\dagger}}
\newcommand{\diag}[1]{\mathrm{diag}\left({#1}\right)}
\newcommand{\blockdiag}[1]{\mbox{\rm bdiag}\left({#1}\right)}
\newcommand{\tran}{^{\top}}
\newcommand{\inner}[1]{\langle {#1} \rangle}
\newcommand{\ped}[1]{_{\mathrm{#1}}}
\newcommand{\ap}[1]{^{\mathrm{#1}}}

\newcommand{\blu}[1]{\textcolor{blue}{#1}}
\newcommand{\red}[1]{\textcolor{red}{#1}}
\newcommand{\green}[1]{\textcolor{green}{#1}}
\newcommand{\cyan}[1]{\textcolor{cyan}{#1}}

\newcommand{\beq}{\begin{equation}}
\newcommand{\eeq}{\end{equation}}
\newcommand{\bea}{\begin{eqnarray}}
\newcommand{\eea}{\end{eqnarray}}
\newcommand{\beas}{\begin{eqnarray*}}
\newcommand{\eeas}{\end{eqnarray*}}
\newcommand{\ba}{\begin{array}}
\newcommand{\ea}{\end{array}}
\newcommand{\bit}{\begin{itemize}}
\newcommand{\eit}{\end{itemize}}
\newcommand{\ben}{\begin{enumerate}}
\newcommand{\een}{\end{enumerate}}
\newcommand{\bde}{\begin{description}}
\newcommand{\ede}{\end{description}}
\newcommand{\bsp}{\begin{split}}
\newcommand{\esp}{\end{split}}

\newtheorem{corollary}{Corollary}
\newtheorem{theorem}{Theorem}
\newtheorem{exercise}{Exercise}
\newtheorem{solution}{Solution}
\newtheorem{assumption}{Assumption}
\newtheorem{definition}{Definition}
\newtheorem{proposition}{Proposition}
\newtheorem{lemma}{Lemma}
\newtheorem{fact}{Fact}

%
%

\def\nocolon{}

\newcommand{\monthyear}{%
  \ifcase\month\or January\or February\or March\or April\or May\or June\or
  July\or August\or September\or October\or November\or
  December\fi\space\number\year
}

\newcommand{\openepigraph}[2]{%
  \begin{fullwidth}
  \sffamily\large
  \begin{doublespace}
  \noindent\allcaps{#1}\\
  \noindent\allcaps{#2}
  \end{doublespace}
  \end{fullwidth}
}

\newcommand{\blankpage}{\newpage\hbox{}\thispagestyle{empty}\newpage}

\begin{abstract} 

For many applications of reinforcement learning it can be more convenient to specify both a reward function and constraints, rather than trying to design behavior through the reward function. For example, systems that physically interact with or around humans should satisfy safety constraints. Recent advances in policy search algorithms \cite{Mnih2016,Schulman2006,Lillicrap2016,Levine2016} have enabled new capabilities in high-dimensional control, but do not consider the constrained setting.

We propose Constrained Policy Optimization (CPO), the first general-purpose policy search algorithm for constrained reinforcement learning with guarantees for near-constraint satisfaction at each iteration. Our method allows us to train neural network policies for high-dimensional control while making guarantees about policy behavior all throughout training. Our guarantees are based on a new theoretical result, which is of independent interest: we prove a bound relating the expected returns of two policies to an average divergence between them. We demonstrate the effectiveness of our approach on simulated robot locomotion tasks where the agent must satisfy constraints motivated by safety.

\end{abstract} 

\section{Introduction}

Recently, deep reinforcement learning has enabled neural network policies to achieve state-of-the-art performance on many high-dimensional control tasks, including Atari games (using pixels as inputs) \cite{Mnih2015, Mnih2016}, robot locomotion and manipulation \cite{Schulman2006, Levine2016, Lillicrap2016}, and even Go at the human grandmaster level \cite{Silver2016}. 

In reinforcement learning (RL), agents learn to act by trial and error, gradually improving their performance at the task as learning progresses. Recent work in deep RL assumes that agents are free to explore \emph{any behavior} during learning, so long as it leads to performance improvement. In many realistic domains, however, it may be unacceptable to give an agent complete freedom. Consider, for example, an industrial robot arm learning to assemble a new product in a factory. Some behaviors could cause it to damage itself or the plant around it---or worse, take actions that are harmful to people working nearby. In domains like this, \emph{safe exploration} for RL agents is important \cite{Moldovan2012, Amodei2016}. A natural way to incorporate safety is via constraints.

A standard and well-studied formulation for reinforcement learning with constraints is the constrained Markov Decision Process (CMDP) framework \cite{Altman1999}, where agents must satisfy constraints on expectations of auxilliary costs. Although optimal policies for finite CMDPs with known models can be obtained by linear programming, methods for high-dimensional control are lacking.  

Currently, policy search algorithms enjoy state-of-the-art performance on high-dimensional control tasks \cite{Mnih2016, Duan2016}. Heuristic algorithms for policy search in CMDPs have been proposed \cite{Uchibe2007}, and approaches based on primal-dual methods can be shown to converge to constraint-satisfying policies \cite{Chow2015}, but there is currently no approach for policy search in continuous CMDPs that guarantees every policy during learning will satisfy constraints. In this work, we propose the first such algorithm, allowing applications to constrained deep RL. 

Driving our approach is a new theoretical result that bounds the difference between the rewards or costs of two different policies. This result, which is of independent interest, tightens known bounds for policy search using trust regions \cite{Kakade2002,Pirotta2013,Schulman2006}, and provides a tighter connection between the theory and practice of policy search for deep RL. Here, we use this result to derive a policy improvement step that guarantees both an increase in reward and satisfaction of constraints on other costs. This step forms the basis for our algorithm, \textit{Constrained Policy Optimization} (CPO), which computes an approximation to the theoretically-justified update. 

In our experiments, we show that CPO can train neural network policies with thousands of parameters on high-dimensional simulated robot locomotion tasks to maximize rewards while successfully enforcing constraints.

\section{Related Work}

Safety has long been a topic of interest in RL research, and a comprehensive overview of safety in RL was given by \cite{Garcia2015}.

Safe policy search methods have been proposed in prior work. Uchibe and Doya \yrcite{Uchibe2007} gave a policy gradient algorithm that uses gradient projection to enforce active constraints, but this approach suffers from an inability to prevent a policy from becoming unsafe in the first place. Bou Ammar et al. \yrcite{BouAmmar2015} propose a theoretically-motivated policy gradient method for lifelong learning with safety constraints, but their method involves an expensive inner loop optimization of a semi-definite program, making it unsuited for the deep RL setting. Their method also assumes that safety constraints are linear in policy parameters, which is limiting. Chow et al. \yrcite{Chow2015} propose a primal-dual subgradient method for risk-constrained reinforcement learning which takes policy gradient steps on an objective that trades off return with risk, while simultaneously learning the trade-off coefficients (dual variables).  

Some approaches specifically focus on application to the deep RL setting. Held et al. \yrcite{Held2017} study the problem for robotic manipulation, but the assumptions they make restrict the applicability of their methods. Lipton et al. \yrcite{Lipton2017} use an `intrinsic fear' heuristic, as opposed to constraints, to motivate agents to avoid rare but catastrophic events. Shalev-Shwartz et al. \yrcite{Shalev-Shwartz2016} avoid the problem of enforcing constraints on parametrized policies by decomposing `desires' from trajectory planning; the neural network policy learns desires for behavior, while the trajectory planning algorithm (which is not learned) selects final behavior and enforces safety constraints. 

In contrast to prior work, our method is the first policy search algorithm for CMDPs that both 1) guarantees constraint satisfaction throughout training, and 2) works for arbitrary policy classes (including neural networks).

\section{Preliminaries}

A Markov decision process (MDP) is a tuple, ($S,A,R,P,\mu$), where $S$ is the set of states, $A$ is the set of actions, $R : S \times A \times S \to \Real{}$ is the reward function, $P : S \times A \times S \to [0,1]$ is the transition probability function (where $P(s'|s,a)$ is the probability of transitioning to state $s'$ given that the previous state was $s$ and the agent took action $a$ in $s$), and $\mu : S \to [0,1]$ is the starting state distribution. A stationary policy $\pi : S \to \calP(A)$ is a map from states to probability distributions over actions, with $\pi(a|s)$ denoting the probability of selecting action $a$ in state $s$. We denote the set of all stationary policies by $\Pi$. 

In reinforcement learning, we aim to select a policy $\pi$ which maximizes a performance measure, $J(\pi)$, which is typically taken to be the infinite horizon discounted total return, $J(\pi) \doteq \tauE{\pi}{\sum_{t=0}^{\infty} \gamma^t R(s_t, a_t, s_{t+1})}$. Here $\gamma \in [0,1)$ is the discount factor, $\tau$ denotes a trajectory ($\tau = (s_0, a_0, s_1, ...)$), and $\tau \sim \pi$ is shorthand for indicating that the distribution over trajectories depends on $\pi$: $s_0 \sim \mu$, $a_t \sim \pi(\cdot|s_t)$, $s_{t+1} \sim P(\cdot | s_t, a_t)$.

Letting $R(\tau)$ denote the discounted return of a trajectory, we express the on-policy value function as $V^{\pi} (s) \doteq \E_{\tau \sim \pi}[R(\tau) | s_0 = s]$ and the on-policy action-value function as $Q^{\pi} (s,a) \doteq \E_{\tau \sim \pi} [R(\tau) | s_0= s, a_0 = a]$. The advantage function is $A^{\pi} (s,a) \doteq Q^{\pi}(s,a) - V^{\pi}(s)$.

Also of interest is the discounted future state distribution, $d^{\pi}$, defined by $d^{\pi} (s) = (1-\gamma) \sum_{t=0}^{\infty} \gamma^t P(s_t = s | \pi)$. It allows us to compactly express the difference in performance between two policies $\pi', \pi$ as
\begin{equation}
J(\pi') - J(\pi) = \frac{1}{1-\gamma} \underset{\begin{subarray}{c} s \sim d^{\pi'} \\ a \sim \pi' \end{subarray}}{\E} \left[ A^{\pi} (s,a) \right], \label{connection}
\end{equation}
where by $a \sim \pi'$, we mean $a \sim \pi'(\cdot |s)$, with explicit notation dropped to reduce clutter. For proof of (\ref{connection}), see \cite{Kakade2002} or Section \ref{appendix} in the supplementary material. 

\section{Constrained Markov Decision Processes}

A constrained Markov decision process (CMDP) is an MDP augmented with constraints that restrict the set of allowable policies for that MDP. Specifically, we augment the MDP with a set $C$ of auxiliary cost functions, $C_1, ..., C_m$ (with each one a function $C_i : S \times A \times S \to \Real{}$ mapping transition tuples to costs, like the usual reward), and limits $d_1, ..., d_m$. Let $J_{C_i} (\pi)$ denote the expected discounted return of policy $\pi$ with respect to cost function $C_i$: $J_{C_i} (\pi) = \tauE{\pi}{\sum_{t=0}^{\infty} \gamma^t C_i (s_t, a_t, s_{t+1})}$. The set of feasible stationary policies for a CMDP is then
\begin{equation*}
\Pi_{C} \doteq \left\{\pi \in \Pi \; : \; \forall i,  J_{C_i} (\pi) \leq d_i\right\},
\end{equation*}
and the reinforcement learning problem in a CMDP is
\begin{equation*}
\pi^* = \arg\max_{ \pi \in \Pi_{C} } J(\pi).
\end{equation*}
The choice of optimizing only over stationary policies is justified: it has been shown that the set of all optimal policies for a CMDP includes stationary policies, under mild technical conditions. For a thorough review of CMDPs and CMDP theory, we refer the reader to \cite{Altman1999}. 

We refer to $J_{C_i}$ as a \textit{constraint return}, or $C_i$-return for short. Lastly, we define on-policy value functions, action-value functions, and advantage functions for the auxiliary costs in analogy to $V^{\pi}$, $Q^{\pi}$, and $A^{\pi}$, with $C_i$ replacing $R$: respectively, we denote these by $V_{C_i}^{\pi}$, $Q_{C_i}^{\pi}$, and $A_{C_i}^{\pi}$.

\section{Constrained Policy Optimization}

For large or continuous MDPs, solving for the exact optimal policy is intractable due to the curse of dimensionality \cite{Sutton1998}. Policy search algorithms approach this problem by searching for the optimal policy within a set $\Pi_{\theta} \subseteq \Pi$  of parametrized policies with parameters $\theta$ (for example, neural networks of a fixed architecture). In local policy search \cite{Peters2008a}, the policy is iteratively updated by maximizing $J(\pi)$ over a local neighborhood of the most recent iterate $\pi_k$: 
\begin{equation} \label{polopt}
\begin{aligned} 
\pi_{k+1} = \arg \max_{\pi \in \Pi_{\theta}} \; & J(\pi) \\
\text{s.t. } & D (\pi, \pi_k) \leq \delta,
\end{aligned}
\end{equation}
where $D$ is some distance measure, and $\delta > 0$ is a step size. When the objective is estimated by linearizing around $\pi_k$ as $J(\pi_k) + g^T (\theta - \theta_k)$, $g$ is the policy gradient, and the standard policy gradient update is obtained by choosing $D(\pi, \pi_k)= \|\theta - \theta_k\|_2$ \cite{Schulman2006}.

In local policy search for CMDPs, we additionally require policy iterates to be feasible for the CMDP, so instead of optimizing over $\Pi_{\theta}$, we optimize over $\Pi_{\theta} \cap \Pi_C$:
\begin{equation}\label{cpolopt}
\begin{aligned}
\pi_{k+1} = \arg \max_{\pi \in \Pi_{\theta}} \; &  J(\pi) \\
\text{s.t. } & J_{C_i} (\pi) \leq d_i \;\;\; i=1, ..., m\\
& D(\pi, \pi_k) \leq \delta.
\end{aligned}
\end{equation}
This update is difficult to implement in practice because it requires evaluation of the constraint functions to determine whether a proposed point $\pi$ is feasible. When using sampling to compute policy updates, as is typically done in high-dimensional control \cite{Duan2016}, this requires off-policy evaluation, which is known to be challenging \cite{Jiang2015}. In this work, we take a different approach, motivated by recent methods for trust region optimization \cite{Schulman2006}.

We develop a principled approximation to (\ref{cpolopt}) with a particular choice of $D$, where we replace the objective and constraints with \textit{surrogate} functions. The surrogates we choose are easy to estimate from samples collected on $\pi_k$, and are good local approximations for the objective and constraints. Our theoretical analysis shows that for our choices of surrogates, we can bound our update's worst-case performance and worst-case constraint violation with values that depend on a hyperparameter of the algorithm. 

To prove the performance guarantees associated with our surrogates, we first prove new bounds on the difference in returns (or constraint returns) between two arbitrary stochastic policies in terms of an average divergence between them. We then show how our bounds permit a new analysis of trust region methods in general: specifically, we prove a worst-case performance degradation at each update.  We conclude by motivating, presenting, and proving gurantees on our algorithm, Constrained Policy Optimization (CPO), a trust region method for CMDPs.

\subsection{Policy Performance Bounds}

In this section, we present the theoretical foundation for our approach---a new bound on the difference in returns between two arbitrary policies. This result, which is of independent interest, extends the works of  \cite{Kakade2002}, \cite{Pirotta2013}, and  \cite{Schulman2006}, providing tighter bounds. As we show later, it also relates the theoretical bounds for trust region policy improvement with the actual trust region algorithms that have been demonstrated to be successful in practice \cite{Duan2016}. In the context of constrained policy search, we later use our results to propose policy updates that both improve the expected return and satisfy constraints. 

The following theorem connects the difference in returns (or constraint returns) between two arbitrary policies to an average divergence between them.
\begin{restatable}{theorem}{performancebound}
\label{maintheorem_0}  For any function $f: S \to \Real{}$ and any policies $\pi'$ and $\pi$, define $\delta_f (s,a,s') \doteq R(s,a,s') + \gamma f(s') - f(s)$,
\begin{equation*}
\epsilon_f^{\pi'} \doteq \max_s \left| \E_{a \sim \pi', s'\sim P} [\delta_f (s,a,s')] \right|, 
\end{equation*}
\begin{align*}
&L_{\pi,f} (\pi') \doteq \underset{\begin{subarray}{c} s \sim d^{\pi} \\ a \sim \pi \\ s' \sim P\end{subarray}}{\E} \left[ \left(\frac{\pi'(a|s)}{\pi(a|s)} - 1 \right) \delta_f(s,a,s') \right], \text{ and }\label{surrogate_0}
\end{align*}
\begin{equation*}
\begin{aligned}
&D_{\pi,f}^{\pm} (\pi') \doteq  \frac{L_{\pi,f} (\pi')}{1-\gamma} \pm \frac{2\gamma \epsilon_f^{\pi'}}{(1-\gamma)^2} \underset{s \sim d^{\pi}}{\E} \left[ D_{TV} (\pi'||\pi)[s] \right],
\end{aligned} 
\end{equation*}
%
where $D_{TV}(\pi'||\pi)[s] = (1/2)\sum_a \left| \pi'(a|s) - \pi(a|s) \right|$ is the total variational divergence between action distributions at $s$. The following bounds hold:
\begin{equation}
D_{\pi,f}^{+} (\pi') \geq J(\pi') - J(\pi) \geq D_{\pi,f}^{-} (\pi'). \label{bound}
\end{equation} 
Furthermore, the bounds are tight (when $\pi' = \pi$, all three expressions are identically zero). 
\end{restatable}

Before proceeding, we connect this result to prior work. By bounding the expectation $\E_{s \sim d^{\pi}}\left[D_{TV} (\pi' || \pi)[s]\right]$ with $\max_s D_{TV} (\pi'||\pi)[s]$, picking $f = V^{\pi}$, and bounding $\epsilon_{V^{\pi}}^{\pi'}$ to get a second factor of $\max_s D_{TV} (\pi'||\pi)[s]$, we recover (up to assumption-dependent factors) the bounds given by Pirotta et~al. \yrcite{Pirotta2013} as Corollary 3.6, and by Schulman et~al. \yrcite{Schulman2006} as Theorem 1a. 

The choice of $f = V^{\pi}$ allows a useful form of the lower  bound, so we give it as a corollary.
\begin{corollary}
For any policies $\pi', \pi$, with $\epsilon^{\pi'} \doteq \max_s | \E_{a \sim \pi'} [A^{\pi} (s,a) ] |$, the following bound holds:
\begin{equation}
\begin{aligned}
&J(\pi') - J(\pi)\\
& \geq \frac{1}{1-\gamma} \underset{\begin{subarray}{c} s \sim d^{\pi} \\ a \sim \pi' \end{subarray}}{\E} \left[ A^{\pi} (s,a) - \frac{2\gamma \epsilon^{\pi'}}{1-\gamma}  D_{TV} (\pi'||\pi)[s] \right]. \label{bound2} 
\end{aligned}
\end{equation}
\label{advantage-bound}
\end{corollary}

The bound (\ref{bound2}) should be compared with equation (\ref{connection}). The term $(1-\gamma)^{-1} \E_{s\sim d^{\pi}, a\sim \pi'}[A^{\pi}(s,a)]$ in (\ref{bound2}) is an approximation to $J(\pi') - J(\pi)$, using the state distribution $d^{\pi}$ instead of $d^{\pi'}$, which is known to equal $J(\pi') - J(\pi)$ to first order in the parameters of $\pi'$ on a neighborhood around $\pi$ \cite{Kakade2002}. The bound can therefore be viewed as describing the worst-case approximation error, and it justifies using the approximation as a \textit{surrogate} for $J(\pi') - J(\pi)$. 

Equivalent expressions for the auxiliary costs, based on the upper bound, also follow immediately; we will later use them to make guarantees for the safety of CPO.
\begin{corollary}
For any policies $\pi', \pi$, and any cost function $C_i$, with $\epsilon^{\pi'}_{C_i} \doteq \max_s | \E_{a \sim \pi'} [A^{\pi}_{C_i} (s,a) ] |$, the following bound holds:
\begin{equation}
\begin{aligned}
&J_{C_i} (\pi') - J_{C_i} (\pi)\\
& \leq \frac{1}{1-\gamma} \underset{\begin{subarray}{c} s \sim d^{\pi} \\ a \sim \pi' \end{subarray}}{\E} \left[ A^{\pi}_{C_i} (s,a) + \frac{2\gamma \epsilon^{\pi'}_{C_i}}{1-\gamma}  D_{TV} (\pi'||\pi)[s] \right]. \label{bound3} 
\end{aligned}
\end{equation}
\label{safety-advantage-bound}
\end{corollary}

The bounds we have given so far are in terms of the TV-divergence between policies, but trust region methods constrain the KL-divergence between policies, so bounds that connect performance to the KL-divergence are desirable. We make the connection through Pinsker's inequality \cite{Csiszar1981}: for arbitrary distributions $p,q$, the TV-divergence and KL-divergence are related by $D_{TV} (p||q) \leq \sqrt{D_{KL} (p||q) /2}$. Combining this with Jensen's inequality, we obtain
\begin{align}
\underset{s \sim d^{\pi}}{\E}\left[D_{TV}(\pi'||\pi)[s]\right] &\leq \underset{s \sim d^{\pi}}{\E}\left[\sqrt{\frac{1}{2} D_{KL}(\pi'||\pi)[s]}\right] \nonumber \\
&\leq \sqrt{\frac{1}{2}  \underset{s \sim d^{\pi}}{\E}\left[D_{KL}(\pi'||\pi)[s]\right]} \label{tvkl}
\end{align}

From (\ref{tvkl}) we immediately obtain the following.
\begin{corollary} \label{klbound}
In bounds (\ref{bound}), (\ref{bound2}), and (\ref{bound3}), make the substitution
\begin{equation*}
\underset{s \sim d^{\pi}}{\E}\left[D_{TV}(\pi'||\pi)[s]\right] \to \sqrt{\frac{1}{2}  \underset{s \sim d^{\pi}}{\E}\left[D_{KL}(\pi'||\pi)[s]\right]}.
\end{equation*}
The resulting bounds hold. 
\end{corollary}

\subsection{Trust Region Methods}

Trust region algorithms for reinforcement learning \cite{Schulman2006, Schulman2016} have policy updates of the form
\begin{equation}
\begin{aligned}
\pi_{k+1} = \arg \max_{\pi \in \Pi_{\theta}} \; & \underset{\begin{subarray}{c}s \sim d^{\pi_k} \\ a \sim \pi\end{subarray}}{\E}\left[A^{\pi_k}(s,a)\right] \\
\text{s.t. } &\bar{D}_{KL} (\pi || \pi_k) \leq \delta,
\end{aligned} \label{trpo}
\end{equation}
where $\bar{D}_{KL}(\pi || \pi_k) = \E_{s\sim\pi_k}\left[D_{KL}(\pi||\pi_k)[s] \right]$, and $\delta > 0$ is the step size. The set $\{\pi_{\theta} \in \Pi_{\theta} : \bar{D}_{KL} (\pi||\pi_k) \leq \delta\}$ is called the \textit{trust region}. 

The primary motivation for this update is that it is an approximation to optimizing the lower bound on policy performance given in (\ref{bound2}), which would guarantee monotonic performance improvements. This is important for optimizing neural network policies, which are known to suffer from performance collapse after bad updates \cite{Duan2016}. Despite the approximation, trust region steps usually give monotonic improvements \cite{Schulman2006, Duan2016} and have shown state-of-the-art performance in the deep RL setting \cite{Duan2016, Gu2017}, making the approach appealing for developing policy search methods for CMDPs.

Until now, the particular choice of trust region for (\ref{trpo}) was heuristically motivated; with (\ref{bound2}) and Corollary \ref{klbound}, we are able to show that it is principled and comes with a worst-case performance degradation guarantee that depends on $\delta$.

\begin{proposition}[Trust Region Update Performance] \label{trpoprop} Suppose $\pi_k$, $\pi_{k+1}$ are related by (\ref{trpo}), and that $\pi_k \in \Pi_{\theta}$. A lower bound on the policy performance difference between $\pi_k$ and $\pi_{k+1}$ is
\begin{equation}
J(\pi_{k+1}) - J(\pi_k) \geq \frac{ -\sqrt{2\delta} \gamma \epsilon^{\pi_{k+1}}}{(1 -\gamma)^2}, \label{trpobound}
\end{equation}
where $\epsilon^{\pi_{k+1}} = \max_s \left| \E_{a \sim \pi_{k+1}} \left[A^{\pi_k}(s,a)\right]\right|$.
\end{proposition}

\begin{proof}
$\pi_k$ is a feasible point of (\ref{trpo}) with objective value 0, so $\E_{s \sim d^{\pi_k}, a \sim \pi_{k+1}} \left[A^{\pi_k}(s,a)\right] \geq 0$. The rest follows by (\ref{bound2}) and Corollary \ref{klbound}, noting that (\ref{trpo}) bounds the average KL-divergence by $\delta$. 
\end{proof}

This result is useful for two reasons: 1) it is of independent interest, as it helps tighten the connection between theory and practice for deep RL, and 2) the choice to develop CPO as a trust region method means that CPO inherits this performance guarantee.

\subsection{Trust Region Optimization for Constrained MDPs}

\textit{Constrained policy optimization} (CPO), which we present and justify in this section, is a policy search algorithm for CMDPs with updates that approximately solve (\ref{cpolopt}) with a particular choice of $D$. First, we describe a policy search update for CMDPs that alleviates the issue of off-policy evaluation, and comes with guarantees of monotonic performance improvement and constraint satisfaction. Then, because the theoretically guaranteed update will take too-small steps in practice, we propose CPO as a practical approximation based on trust region methods. 

By corollaries \ref{advantage-bound}, \ref{safety-advantage-bound}, and \ref{klbound}, for appropriate coefficients $\alpha_k$, $\beta_k^i$ the update
\begin{equation*}\label{cpo_ancestor}
\begin{aligned}
& \pi_{k+1} = \arg \max_{\pi \in \Pi_{\theta}} \;   \smallsurr{\pi_k}{\pi}{A^{\pi_k}(s,a)} - \alpha_k \sqrt{\bar{D}_{KL}(\pi || \pi_k)}  \\
& \text{s.t. } J_{C_i} (\pi_k) + \smallsurr{\pi_k}{\pi}{\frac{A^{\pi_k}_{C_i} (s,a)}{1-\gamma}} + \beta_k^i \sqrt{\bar{D}_{KL}(\pi || \pi_k)} \leq d_i 
\end{aligned}
\end{equation*}
is guaranteed to produce policies with monotonically nondecreasing returns that satisfy the original constraints. (Observe that the constraint here is on an upper bound for $J_{C_i}(\pi)$ by (\ref{bound3}).) The off-policy evaluation issue is alleviated, because both the objective and constraints involve expectations over state distributions $d^{\pi_k}$, which we presume to have samples from. Because the bounds are tight, the problem is always feasible (as long as $\pi_0$ is feasible). However, the penalties on policy divergence are quite steep for discount factors close to 1, so steps taken with this update might be small.

Inspired by trust region methods, we propose CPO, which uses a trust region instead of penalties on policy divergence to enable larger step sizes:
\begin{equation}\label{cpo}
\begin{aligned}
\pi_{k+1} = &\arg \max_{\pi \in \Pi_{\theta}} \;   \smallsurr{\pi_k}{\pi}{A^{\pi_k}(s,a)}  \\
\text{s.t. } & J_{C_i} (\pi_k) + \surrA{\pi_k}{\pi}{A^{\pi_k}_{C_i} (s,a)} \leq d_i \;\;\; \forall i\\
& \bar{D}_{KL} (\pi || \pi_k) \leq \delta.
\end{aligned}
\end{equation}
Because this is a trust region method, it inherits the performance guarantee of Proposition \ref{trpoprop}. Furthermore, by corollaries \ref{safety-advantage-bound} and \ref{klbound}, we have a performance guarantee for approximate satisfaction of constraints:

\begin{proposition}[CPO Update Worst-Case Constraint Violation] \label{cpobound} Suppose $\pi_k, \pi_{k+1}$ are related by (\ref{cpo}), and that $\Pi_{\theta}$ in (\ref{cpo}) is any set of policies with $\pi_k \in \Pi_{\theta}$. An upper bound on the $C_i$-return of $\pi_{k+1}$ is
\begin{equation*}
J_{C_i} (\pi_{k+1}) \leq d_i + \frac{\sqrt{2\delta} \gamma \epsilon^{\pi_{k+1}}_{C_i}}{(1 -\gamma)^2}, 
\end{equation*}
where $\epsilon^{\pi_{k+1}}_{C_i} = \max_s \left| \E_{a \sim \pi_{k+1}} \left[A^{\pi_k}_{C_i} (s,a)\right]\right|$.

\end{proposition}

\section{Practical Implementation}

In this section, we show how to implement an approximation to the update (\ref{cpo}) that can be efficiently computed, even when optimizing policies with thousands of parameters. To address the issue of approximation and sampling errors that arise in practice, as well as the potential violations described by Proposition \ref{cpobound}, we also propose to tighten the constraints by constraining upper bounds of the auxilliary costs, instead of the auxilliary costs themselves.

\subsection{Approximately Solving the CPO Update}

For policies with high-dimensional parameter spaces like neural networks, (\ref{cpo}) can be impractical to solve directly because of the computational cost. However, for small step sizes $\delta$, the objective and cost constraints are well-approximated by linearizing around $\pi_k$, and the KL-divergence constraint is well-approximated by second order expansion (at $\pi_k = \pi$, the KL-divergence and its gradient are both zero). Denoting the gradient of the objective as $g$, the gradient of constraint $i$ as $b_i$, the Hessian of the KL-divergence as $H$,  and defining $c_i \doteq J_{C_i}(\pi_k) - d_i$, the approximation to (\ref{cpo}) is:
\begin{equation}
\begin{aligned}
\theta_{k+1} =\arg\max_{\theta} \;\;\;&  g^T (\theta - \theta_k)  \\
\text{s.t. } \;\;\;& c_i + b_i^T (\theta - \theta_k) \leq 0 \;\;\;  i=1,...,m \\
& \frac{1}{2} (\theta - \theta_k)^T H (\theta - \theta_k) \leq \delta.
\end{aligned} \label{cpoapprox}
\end{equation}
Because the Fisher information matrix (FIM) $H$ is always positive semi-definite (and we will assume it to be positive-definite in what follows), this optimization problem is convex and, when feasible, can be solved efficiently using duality. (We reserve the case where it is not feasible for the next subsection.) With $B \doteq [b_1 , ..., b_m]$ and $c \doteq [c_1, ..., c_m]^T$, a dual to (\ref{cpoapprox}) can be expressed as
\begin{equation}
\begin{aligned}
\max_{\begin{subarray}{c} \lambda \geq 0 \\ \nu \succeq 0\end{subarray}} \frac{-1}{2\lambda} \left( g^T H^{-1} g - 2 r^T \nu + \nu^T S \nu\right) + \nu^T c - \frac{\lambda \delta}{2},
\end{aligned} \label{cpodual}
\end{equation}
where $r\doteq g^T H^{-1} B$, $S\doteq B^T H^{-1} B$. This is a convex program in $m+1$ variables; when the number of constraints is small by comparison to the dimension of $\theta$, this is much easier to solve than (\ref{cpoapprox}). If $\lambda^*, \nu^*$ are a solution to the dual, the solution to the primal is
\begin{equation}
\theta^* = \theta_k + \frac{1}{\lambda^*} H^{-1} \left( g - B \nu^*\right). \label{proposed}
\end{equation}
Our algorithm solves the dual for $\lambda^*, \nu^*$ and uses it to propose the policy update (\ref{proposed}). For the special case where there is only one constraint, we give an analytical solution in the supplementary material (Theorem \ref{thmlqclp}) which removes the need for an inner-loop optimization. Our experiments have only  a single constraint, and make use of the analytical solution. 

Because of approximation error, the proposed update may not satisfy the constraints in (\ref{cpo}); a backtracking line search is used to ensure surrogate constraint satisfaction. Also, for high-dimensional policies, it is impractically expensive to invert the FIM. This poses a challenge for computing $H^{-1}g$ and $H^{-1} b_i$, which appear in the dual. Like \cite{Schulman2006}, we approximately compute them using the conjugate gradient method.

%
%

\subsection{Feasibility}

Due to approximation errors, CPO may take a bad step and produce an infeasible iterate $\pi_k$. Sometimes (\ref{cpoapprox}) will still be feasible and CPO can automatically recover from its bad step, but for the infeasible case, a recovery method is necessary. In our experiments, where we only have one constraint, we recover by proposing an update to purely decrease the constraint value:
\begin{equation}
\theta^* = \theta_k - \sqrt{\frac{2\delta}{b^T H^{-1} b}} H^{-1} b. \label{infeasibleproposal}
\end{equation}
As before, this is followed by a line search. This approach is principled in that it uses the limiting search direction as the intersection of the trust region and the constraint region shrinks to zero. We give the pseudocode for our algorithm (for the single-constraint case) as Algorithm \ref{alg1}. 

\begin{algorithm}[tb]
   \caption{Constrained Policy Optimization}
   \label{alg1}
\begin{algorithmic}
   \STATE {\bfseries Input:} Initial policy $\pi_0 \in \Pi_{\theta}$ tolerance $\alpha$
	 \FOR{$k = 0,1,2,...$} 
	 \STATE Sample a set of trajectories $\calD = \{\tau\} \sim \pi_k = \pi(\theta_k)$
	 \STATE Form sample estimates $\hat{g}, \hat{b}, \hat{H}, \hat{c}$ with $\calD$
	 \IF{approximate CPO is feasible}
	 	\STATE Solve dual problem (\ref{cpodual}) for $\lambda^*_k, \nu^*_k$ 
	 	\STATE Compute policy proposal $\theta^*$ with (\ref{proposed})
	 \ELSE
	 	\STATE Compute recovery policy proposal $\theta^*$ with (\ref{infeasibleproposal}) 
	 \ENDIF
	 \STATE Obtain $\theta_{k+1}$ by backtracking linesearch  to enforce satisfaction of sample estimates of constraints in (\ref{cpo})
	\ENDFOR
\end{algorithmic}
\end{algorithm}

\subsection{Tightening Constraints via Cost Shaping}

Because of the various approximations between (\ref{cpolopt}) and our practical algorithm, it is important to build a factor of safety into the algorithm to minimize the chance of constraint violations. To this end, we choose to constrain upper bounds on the original constraints, $C_i^+$, instead of the original constraints themselves. We do this by cost shaping:
\begin{equation}
C_i^+ (s,a,s') = C_i (s,a,s') + \Delta_i (s,a,s'), \label{safetyfactor}
\end{equation}
where $\Delta_i : S \times A \times S \to \Realp{}$ correlates in some useful way with $C_i$. 

In our experiments, where we have only one constraint, we partition states into \textit{safe states} and \textit{unsafe states}, and the agent suffers a safety cost of $1$ for being in an unsafe state. We choose $\Delta$ to be the probability of entering an unsafe state within a fixed time horizon, according to a learned model that is updated at each iteration. This choice confers the additional benefit of smoothing out sparse constraints.

\section{Connections to Prior Work}

Our method has similar policy updates to primal-dual methods like those proposed by Chow et al. \yrcite{Chow2015}, but crucially, we differ in computing the dual variables (the Lagrange multipliers for the constraints). In primal-dual optimization (PDO), dual variables are stateful and learned concurrently with the primal variables \cite{Boyd2003}. In a  PDO algorithm for solving (\ref{cpolopt}), dual variables would be updated according to
\begin{equation}
\nu_{k+1} = \left(\nu_k + \alpha_k \left(J_C (\pi_k) - d\right)\right)_+, \label{pdodual}
\end{equation} 
where $\alpha_k$ is a learning rate. In this approach, intermediary policies are not guaranteed to satisfy constraints---only the policy at convergence is. By contrast, CPO computes new dual variables from scratch at each update to exactly enforce constraints. 

\begin{figure*}[t]
\centering
Returns:

\begin{subfigure}{.2\textwidth}
  \centering
  \includegraphics[width=\linewidth, trim={0.5cm, 0.5cm, 0.5cm, 0}, clip]{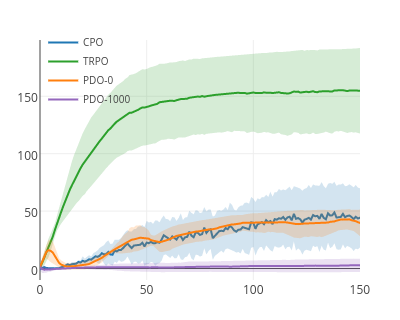}
\end{subfigure}%
\begin{subfigure}{.2\textwidth}
  \centering
  \includegraphics[width=\linewidth, trim={0.5cm, 0.5cm, 0.5cm, 0}, clip]{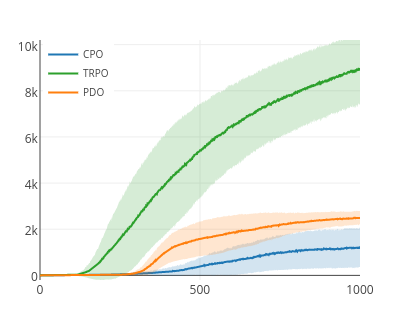}
\end{subfigure}%
\begin{subfigure}{.2\textwidth}
  \centering
  \includegraphics[width=\linewidth, trim={0.4cm, 0.5cm, 0.5cm, 0}, clip]{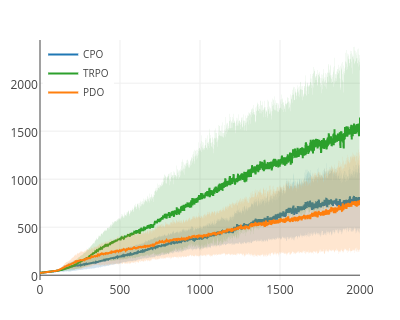}
\end{subfigure}%
\begin{subfigure}{.2\textwidth}
  \centering
  \includegraphics[width=\linewidth, trim={0.5cm, 0.5cm, 0.5cm, 0}, clip]{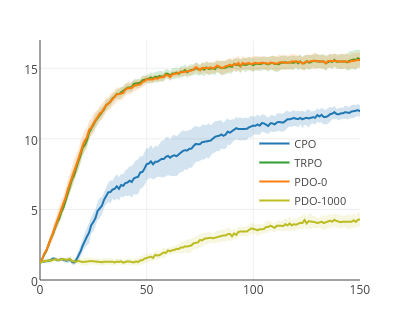}
\end{subfigure}%
\begin{subfigure}{.2\textwidth}
  \centering
  \includegraphics[width=\linewidth, trim={0.5cm, 0.5cm, 0.5cm, 0}, clip]{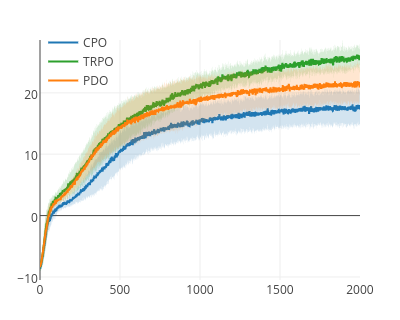}
\end{subfigure}

\rule{\linewidth}{0.5pt}
Constraint values: (closer to the limit is better)

\begin{subfigure}{.2\textwidth}
  \centering
  \includegraphics[width=\linewidth, trim={0.5cm, 0.5cm, 0.5cm, 0}, clip]{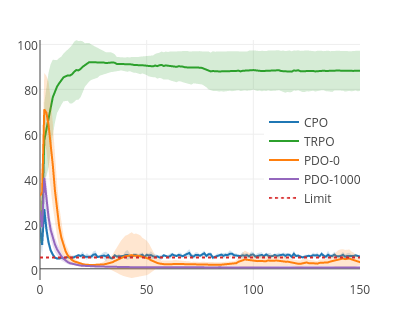}
  \caption{Point-Circle}
\end{subfigure}%
\begin{subfigure}{.2\textwidth}
  \centering
  \includegraphics[width=\linewidth, trim={0.5cm, 0.5cm, 0.5cm, 0}, clip]{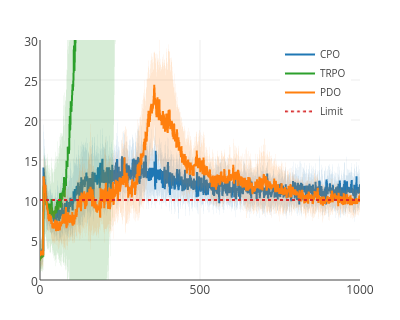}
  \caption{Ant-Circle}
\end{subfigure}%
\begin{subfigure}{.2\textwidth}
  \centering
  \includegraphics[width=\linewidth, trim={0.5cm, 0.5cm, 0.5cm, 0}, clip]{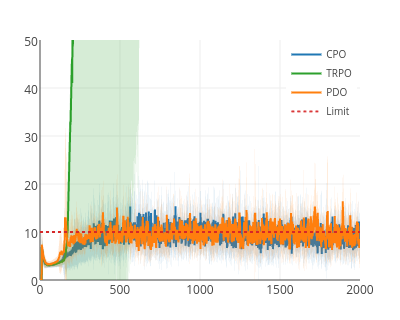}
  \caption{Humanoid-Circle}
\end{subfigure}%
\begin{subfigure}{.2\textwidth}
  \centering
  \includegraphics[width=\linewidth, trim={0.5cm, 0.5cm, 0.5cm, 0}, clip]{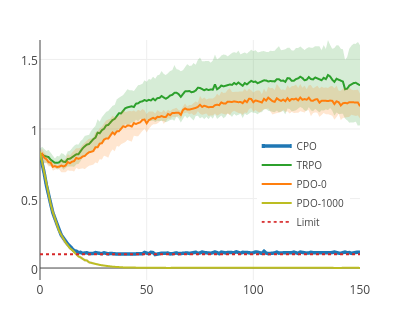}
  \caption{Point-Gather}
\end{subfigure}%
\begin{subfigure}{.2\textwidth}
  \centering
  \includegraphics[width=\linewidth, trim={0.5cm, 0.5cm, 0.5cm, 0}, clip]{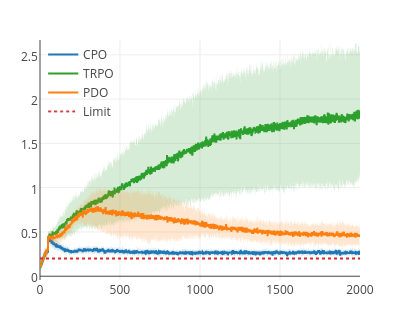}
  \caption{Ant-Gather}
\end{subfigure}
\caption[]{Average performance for CPO, PDO, and TRPO over several seeds (5 in the Point environments, 10 in all others); the $x$-axis is training iteration. CPO drives the constraint function almost directly to the limit in all experiments, while PDO frequently suffers from over- or under-correction. TRPO is included to verify that optimal unconstrained behaviors are infeasible for the constrained problem.}
\label{main_results}
\end{figure*}

\section{Experiments}

In our experiments, we aim to answer the following:
\begin{itemize}
\item Does CPO succeed at enforcing behavioral constraints when training neural network policies with thousands of parameters? 
\item How does CPO compare with a baseline that uses primal-dual optimization? Does CPO behave better with respect to constraints?
\item How much does it help to constrain a cost upper bound (\ref{safetyfactor}), instead of directly constraining the cost?
\item What benefits are conferred by using constraints instead of fixed penalties?
\end{itemize}

We designed experiments that are easy to interpret and motivated by safety. We consider two tasks, and train multiple different agents (robots) for each task:
\begin{itemize}
\item \textbf{Circle}: The agent is rewarded for running in a wide circle, but is constrained to stay within a safe region smaller than the radius of the target circle. 
\item \textbf{Gather}: The agent is rewarded for collecting green apples, and constrained to avoid red bombs. 
\end{itemize}
For the Circle task, the exact geometry is illustrated in Figure \ref{geom} in the supplementary material. Note that there are no physical walls: the agent only interacts with boundaries through the constraint costs. The reward and constraint cost functions are described in supplementary material (Section \ref{envirosec}). In each of these tasks, we have only one constraint; we refer to it as $C$ and its upper bound from (\ref{safetyfactor}) as $C^+$. 

We experiment with three different agents: a point-mass $(S \subseteq \Real{9}, A \subseteq \Real{2})$, a quadruped robot (called an `ant') $(S \subseteq \Real{32}, A \subseteq \Real{8})$, and a simple humanoid $(S \subseteq \Real{102}, A \subseteq \Real{10})$. We train all agent-task combinations except for Humanoid-Gather. 

For all experiments, we use neural network policies with two hidden layers of size $(64,32)$. Our experiments are implemented in rllab \cite{Duan2016}.

\begin{figure}[t]
\centering

\begin{subfigure}{.22\textwidth}
  \centering
  \includegraphics[height=0.9\linewidth]{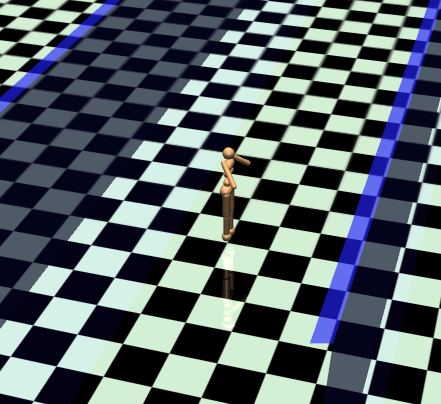}
  \caption{Humanoid-Circle}
\end{subfigure}%
\begin{subfigure}{.22\textwidth}
  \centering
  \includegraphics[height=0.9\linewidth]{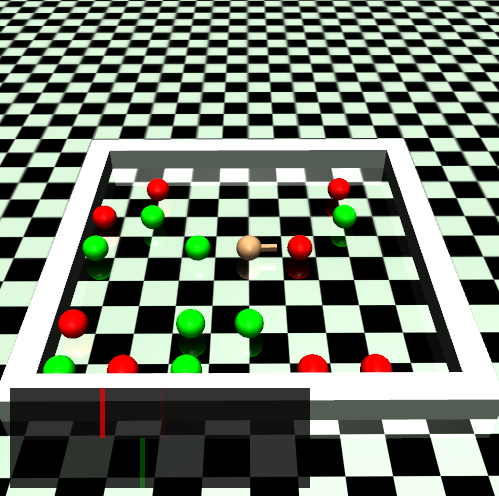}
  \caption{Point-Gather}
\end{subfigure}
\caption[]{The Humanoid-Circle and Point-Gather environments. In Humanoid-Circle, the safe area is between the blue panels. }
\label{graphics}
\end{figure}

\subsection{Evaluating CPO and Comparison Analysis}

Learning curves for CPO and PDO are compiled in Figure \ref{main_results}. Note that we evaluate algorithm performance based on the $C^+$ return, instead of the $C$ return (except for in Point-Gather, where we did not use cost shaping due to that environment's short time horizon), because this is what the algorithm actually constrains in these experiments. 

For our comparison, we implement PDO with (\ref{pdodual}) as the update rule for the dual variables, using a constant learning rate $\alpha$; details are available in supplementary material (Section \ref{pdoimplement}). We emphasize that in order for the comparison to be fair, we give PDO every advantage that is given to CPO, including equivalent trust region policy updates. To benchmark the environments, we also include TRPO (trust region policy optimization) \cite{Schulman2006}, a state-of-the-art \textit{unconstrained} reinforcement learning algorithm. The TRPO experiments show that optimal unconstrained behaviors for these environments are constraint-violating.

We find that CPO is successful at approximately enforcing constraints in all environments. In the simpler environments (Point-Circle and Point-Gather), CPO tracks the constraint return almost \textit{exactly} to the limit value.

By contrast, although PDO usually converges to constraint-satisfying policies in the end, it is not consistently constraint-satisfying throughout training (as expected). For example, see the spike in constraint value that it experiences in Ant-Circle. Additionally, PDO is sensitive to the initialization of the dual variable. By default, we initialize $\nu_0 = 0$, which exploits no prior knowledge about the environment and makes sense when the initial policies are feasible. However, it may seem appealing to set $\nu_0$ high, which would make PDO more conservative with respect to the constraint; PDO could then decrease $\nu$ as necessary after the fact. In the Point environments, we experiment with $\nu_0 = 1000$ and show that although this does assure constraint satisfaction, it also can substantially harm performance with respect to return. Furthermore, we argue that this is not adequate in general: after the dual variable decreases, the agent could learn a new behavior that increases the correct dual variable more quickly than PDO can attain it (as happens in Ant-Circle for PDO; observe that performance is approximately constraint-satisfying until the agent learns how to run at around iteration 350).

We find that CPO generally outperforms PDO on enforcing constraints, without compromising performance with respect to return. CPO quickly stabilizes the constraint return around to the limit value, while PDO is not consistently able to enforce constraints all throughout training.

\subsection{Ablation on Cost Shaping}

In Figure \ref{cs-results}, we compare performance of CPO with and without cost shaping in the constraint. Our metric for comparison is the $C$-return, the `true' constraint. The cost shaping does help, almost completely accounting for CPO's inherent approximation errors. However, CPO is nearly constraint-satisfying even without cost shaping.

\begin{figure}[t]
\centering
\begin{subfigure}{.22\textwidth}
  \centering
  \includegraphics[height=0.9\linewidth]{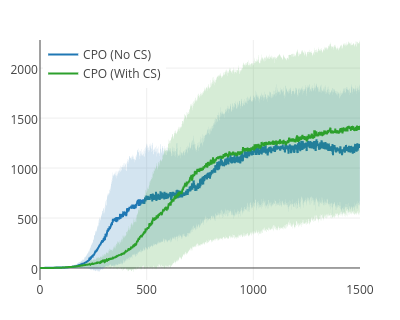}
  \caption{Ant-Circle Return}
\end{subfigure}%
\begin{subfigure}{.22\textwidth}
  \centering
  \includegraphics[height=0.9\linewidth]{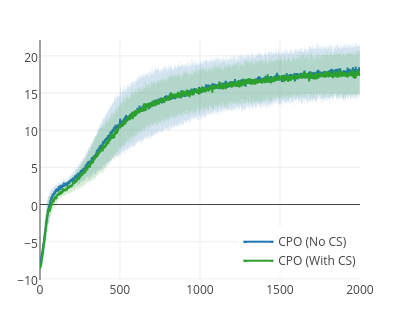}
  \caption{Ant-Gather Return}
\end{subfigure}%

\begin{subfigure}{.22\textwidth}
  \centering
  \includegraphics[height=0.9\linewidth]{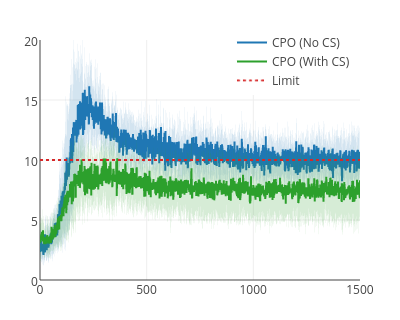}
  \caption{Ant-Circle $C$-Return}
\end{subfigure}%
\begin{subfigure}{.22\textwidth}
  \centering
  \includegraphics[height=0.9\linewidth]{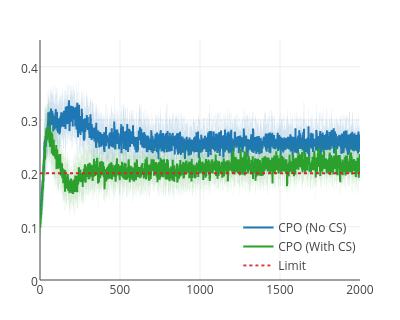}
  \caption{Ant-Gather $C$-Return}
\end{subfigure}%
\caption[]{Using cost shaping (CS) in the constraint while optimizing generally improves the agent's adherence to the true constraint on $C$-return.}
\label{cs-results}
\end{figure}

\subsection{Constraint vs. Fixed Penalty}

\begin{figure}[t]
\centering
\begin{subfigure}{.22\textwidth}
  \centering
  \includegraphics[height=0.9\linewidth]{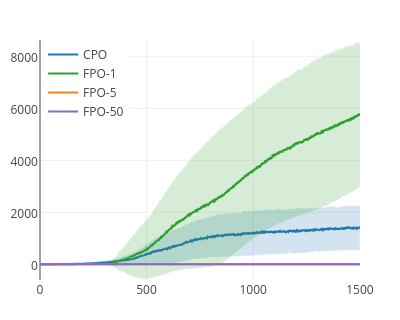}
  \caption{Ant-Circle Return}
\end{subfigure}%
\begin{subfigure}{.22\textwidth}
  \centering
  \includegraphics[height=0.9\linewidth]{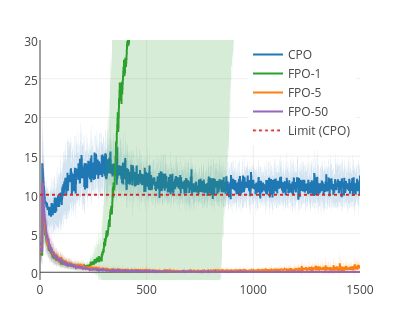}
  \caption{Ant-Circle $C^+$-Return}
\end{subfigure}%
\caption[]{Comparison between CPO and FPO (fixed penalty optimization) for various values of fixed penalty. }
\label{FPO}
\end{figure}

In Figure \ref{FPO}, we compare CPO to a fixed penalty method, where policies are learned using TRPO with rewards  $R(s,a,s') + \lambda C^+ (s,a,s')$ for $\lambda \in \{1, 5, 50\}$. 

We find that fixed penalty methods can be highly sensitive to the choice of penalty coefficient: in Ant-Circle, a penalty coefficient of 1 results in reward-maximizing policies that accumulate massive constraint costs, while a coefficient of 5 (less than an order of magnitude difference) results in cost-minimizing policies that never learn how to acquire any rewards.  In contrast, CPO automatically picks penalty coefficients to attain the desired trade-off between reward and constraint cost.

\section{Discussion}

In this article, we showed that a particular optimization problem results in policy updates that are guaranteed to both improve return and satisfy constraints. This enabled the development of CPO, our policy search algorithm for CMDPs, which approximates the theoretically-guaranteed algorithm in a principled way. We demonstrated that CPO can train neural network policies with thousands of parameters on high-dimensional constrained control tasks, simultaneously maximizing reward and approximately satisfying constraints. Our work represents a step towards applying reinforcement learning in the real world, where constraints on agent behavior are sometimes necessary for the sake of safety.

\section*{Acknowledgements} 

The authors would like to acknowledge Peter Chen, who independently and concurrently derived an equivalent policy improvement bound. 

Joshua Achiam is supported by TRUST (Team for Research in Ubiquitous Secure Technology) which receives support from NSF (award number CCF-0424422). This project also received support from Berkeley Deep Drive and from Siemens.

\bibliography{cpo_bibliography}
\bibliographystyle{icml2017}

\onecolumn
\section{Appendix} \label{appendix}

\subsection{Proof of Policy Performance Bound}
\theoremstyle{remark}
\newtheorem*{remark}{Remark}

\subsubsection{Preliminaries}

Our analysis will make extensive use of the discounted future state distribution, $d^{\pi}$, which is defined as
\begin{equation*}
d^{\pi} (s) = (1-\gamma) \sum_{t=0}^{\infty} \gamma^t P(s_t = s | \pi).
\end{equation*}
It allows us to express the expected discounted total reward  compactly as
\begin{equation}
J(\pi) = \frac{1}{1-\gamma} \underset{\begin{subarray}{c} s \sim d^{\pi} \\ a \sim \pi \\ s' \sim P\end{subarray}}{\E} \left[ R(s,a,s') \right], \label{jpi}
\end{equation}
where by $a \sim \pi$, we mean $a \sim \pi(\cdot|s)$, and by $s' \sim P$, we mean $s' \sim P(\cdot|s,a)$. We drop the explicit notation for the sake of reducing clutter, but it should be clear from context that $a$ and $s'$ depend on $s$. 

First, we examine some useful properties of $d^{\pi}$ that become apparent in vector form for finite state spaces. Let $p^t_{\pi} \in \Real{|S|}$ denote the vector with components $p^t_{\pi} (s) = P(s_t = s | \pi)$, and let $P_{\pi} \in \Real{|S|\times|S|}$ denote the transition matrix with components $P_{\pi} (s'|s) = \int da P(s'|s,a) \pi(a|s)$; then $p^t_{\pi} = P_{\pi} p^{t-1}_{\pi} = P^t_{\pi} \mu$ and 
\begin{eqnarray}
d^{\pi} &=& (1-\gamma) \sum_{t=0}^{\infty} (\gamma P_{\pi})^t \mu \nonumber \\
&=& (1-\gamma) (I - \gamma P_{\pi})^{-1} \mu. \label{dpi}
\end{eqnarray} 

This formulation helps us easily obtain the following lemma.

\begin{lemma} For any function $f : S \to \Real{}$ and any policy $\pi$,
\begin{equation}
(1-\gamma) \underset{s \sim \mu}{\E}\left[f(s)\right] + \underset{\begin{subarray}{c} s \sim d^{\pi} \\ a \sim \pi \\ s' \sim P\end{subarray}}{\E} \left[ \gamma f(s') \right] - \underset{s \sim d^{\pi}}{\E}\left[f(s)\right] = 0. 
\end{equation}
\end{lemma}

\begin{proof} Multiply both sides of (\ref{dpi}) by $(I - \gamma P_{\pi})$ and take the inner product with the vector $f \in \Real{|S|}$. 
\end{proof}

Combining this with (\ref{jpi}), we obtain the following, for any function $f$ and any policy $\pi$:
\begin{equation}
J(\pi) = \underset{s \sim \mu}{\E}[f(s)] + \frac{1}{1-\gamma} \underset{\begin{subarray}{c} s \sim d^{\pi} \\ a \sim \pi \\ s' \sim P\end{subarray}}{\E} \left[R(s,a,s') + \gamma f(s') - f(s)\right]. \label{jpif}
\end{equation}

This identity is nice for two reasons. First: if we pick $f$ to be an approximator of the value function $V^{\pi}$, then (\ref{jpif}) relates the true discounted return of the policy ($J(\pi)$) to the estimate of the policy return ($\E_{s\sim \mu}[f(s)]$) and to the on-policy average TD-error of the approximator; this is aesthetically satisfying. Second: it shows that reward-shaping by $\gamma f(s') - f(s)$ has the effect of translating the total discounted return by $\E_{s\sim \mu} [f(s)]$, a fixed constant independent of policy; this illustrates the finding of Ng. et~al. \yrcite{Ng1999} that reward shaping by $\gamma f(s') + f(s)$ does not change the optimal policy.

It is also helpful to introduce an identity for the vector difference of the discounted future state visitation distributions on two different policies, $\pi'$ and $\pi$. Define the matrices $G \doteq (I - \gamma P_{\pi})^{-1}$, $\bar{G} \doteq (I - \gamma P_{\pi'})^{-1}$, and $\Delta = P_{\pi'} - P_{\pi}$. Then:
\begin{eqnarray*}
G^{-1} - \bar{G}^{-1} &=& (I - \gamma P_{\pi}) - (I - \gamma P_{\pi'}) \\
&=& \gamma \Delta;
\end{eqnarray*} 
left-multiplying by $G$ and right-multiplying by $\bar{G}$, we obtain
\begin{equation*}
\bar{G}  - G = \gamma \bar{G} \Delta G. 
\end{equation*}

Thus
\begin{eqnarray}
d^{\pi'} - d^{\pi} &=& (1-\gamma) \left(\bar{G} - G\right) \mu \nonumber \\
&=& \gamma (1-\gamma) \bar{G} \Delta G \mu \nonumber \\
&=& \gamma \bar{G} \Delta d^{\pi}. \label{dpidiff}
\end{eqnarray}

For simplicity in what follows, we will only consider MDPs with finite state and action spaces, although our attention is on MDPs that are too large for tabular methods. 

\subsubsection{Main Results}

In this section, we will derive and present the new policy improvement bound. We will begin with a lemma:

\begin{lemma}\label{policybound0} For any function $f: S \to \Real{}$ and any policies $\pi'$ and $\pi$, define
\begin{equation}
L_{\pi,f} (\pi') \doteq \underset{\begin{subarray}{c} s \sim d^{\pi} \\ a \sim \pi \\ s' \sim P\end{subarray}}{\E} \left[ \left(\frac{\pi'(a|s)}{\pi(a|s)} - 1 \right) \left(R(s,a,s') + \gamma f(s') - f(s) \right)\right],\label{surrogate}
\end{equation}
and $\epsilon_f^{\pi'} \doteq \max_s \left| \E_{a \sim \pi', s'\sim P} [R(s,a,s') + \gamma f(s') - f(s)] \right|$. Then the following bounds hold:
\begin{align}
J(\pi') - J(\pi) \geq \frac{1}{1-\gamma}\left(L_{\pi,f} (\pi') - 2\epsilon_f^{\pi'} D_{TV} (d^{\pi'} || d^{\pi})\right), \label{bound0} \\
J(\pi') - J(\pi) \leq \frac{1}{1-\gamma}\left(L_{\pi,f} (\pi') + 2\epsilon_f^{\pi'} D_{TV} (d^{\pi'} || d^{\pi})\right), 
\label{bound0b}
\end{align}
where $D_{TV}$ is the total variational divergence. Furthermore, the bounds are tight (when $\pi' = \pi$, the LHS and RHS are identically zero). 
\end{lemma}

\begin{proof} First, for notational convenience, let $\delta_f (s,a,s') \doteq R(s,a,s') + \gamma f(s') - f(s)$. (The choice of $\delta$ to denote this quantity is intentionally suggestive---this bears a strong resemblance to a TD-error.) By (\ref{jpif}), we obtain the identity
\begin{equation*}
J(\pi') - J(\pi) = \frac{1}{1-\gamma} \left(\underset{\begin{subarray}{c} s \sim d^{\pi'} \\ a \sim \pi' \\ s' \sim P\end{subarray}}{\E} \left[ \delta_f (s,a,s') \right] - \underset{\begin{subarray}{c} s \sim d^{\pi} \\ a \sim \pi \\ s' \sim P\end{subarray}}{\E} \left[ \delta_f(s,a,s') \right].\right)
\end{equation*}

Now, we restrict our attention to the first term in this equation. Let $\bar{\delta}_f^{\pi'} \in \Real{|S|}$ denote the vector of components $\bar{\delta}_f^{\pi'} (s) = \E_{a \sim \pi', s' \sim P} [\delta_f(s,a,s') |s]$.  Observe that
\begin{eqnarray*}
\begin{aligned}
\underset{\begin{subarray}{c} s \sim d^{\pi'} \\ a \sim \pi' \\ s' \sim P\end{subarray}}{\E} \left[ \delta_f (s,a,s') \right] & = \left\langle d^{\pi'}, \bar{\delta}_f^{\pi'} \right\rangle\\
& = \left\langle d^{\pi}, \bar{\delta}_f^{\pi'} \right\rangle + \left\langle d^{\pi'} - d^{\pi}, \bar{\delta}_f^{\pi'} \right\rangle
\end{aligned}
\end{eqnarray*}

This term is then straightforwardly bounded by applying H\"{o}lder's inequality; for any $p,q \in [1, \infty]$ such that $1/p + 1/q = 1$, we have
\begin{equation*}
\left\langle d^{\pi}, \bar{\delta}_f^{\pi'} \right\rangle + \left\| d^{\pi'} - d^{\pi}\right\|_p \left\|\bar{\delta}_f^{\pi'} \right\|_q \geq \underset{\begin{subarray}{c} s \sim d^{\pi'} \\ a \sim \pi' \\ s' \sim P\end{subarray}}{\E} \left[ \delta_f (s,a,s') \right] \geq  \left\langle d^{\pi}, \bar{\delta}_f^{\pi'} \right\rangle - \left\| d^{\pi'} - d^{\pi}\right\|_p \left\|\bar{\delta}_f^{\pi'} \right\|_q.  
\end{equation*}
The lower bound leads to (\ref{bound0}), and the upper bound leads to (\ref{bound0b}).

We choose $p=1$ and $q = \infty$; however, we believe that this step is very interesting, and different choices for dealing with the inner product $\left\langle d^{\pi'} - d^{\pi}, \bar{\delta}_f^{\pi'} \right\rangle$ may lead to novel and useful bounds.

With $\left\| d^{\pi'} - d^{\pi}\right\|_1 = 2 D_{TV} (d^{\pi'} || d^{\pi})$ and $\left\|\bar{\delta}_f^{\pi'} \right\|_{\infty} = \epsilon_f^{\pi'}$, the bounds are almost obtained. The last step is to observe that, by the importance sampling identity,
\begin{eqnarray*}
\left\langle d^{\pi}, \bar{\delta}_f^{\pi'} \right\rangle &=& \underset{\begin{subarray}{c} s \sim d^{\pi} \\ a \sim \pi' \\ s' \sim P\end{subarray}}{\E} \left[ \delta_f (s,a,s') \right] \\
&=& \underset{\begin{subarray}{c} s \sim d^{\pi} \\ a \sim \pi \\ s' \sim P\end{subarray}}{\E} \left[\left( \frac{\pi'(a|s)}{\pi(a|s)} \right)\delta_f (s,a,s') \right]. 
\end{eqnarray*}

After grouping terms, the bounds are obtained. 
\end{proof}

This lemma makes use of many ideas that have been explored before; for the special case of $f = V^{\pi}$, this strategy (after bounding $D_{TV} (d^{\pi'} || d^{\pi})$) leads directly to some of the policy improvement bounds previously obtained by Pirotta et~al. and Schulman et~al. The form given here is slightly more general, however, because it allows for freedom in choosing $f$. 

\begin{remark}
It is reasonable to ask if there is a choice of $f$ which maximizes the lower bound here. This turns out to trivially be $f = V^{\pi'}$. Observe that $\E_{s' \sim P} \left[\delta_{V^{\pi'}}(s,a,s') | s,a\right] = A^{\pi'} (s,a)$. For all states, $\E_{a \sim \pi'} [A^{\pi'} (s,a)] = 0$ (by the definition of $A^{\pi'}$), thus $\bar{\delta}^{\pi'}_{V^{\pi'}} = 0$ and $\epsilon_{V^{\pi'}}^{\pi'} = 0$. Also, $L_{\pi, V^{\pi'}} (\pi') = -\E_{s \sim d^{\pi}, a \sim \pi} \left[A^{\pi'}(s,a)\right]$; from (\ref{jpif}) with $f = V^{\pi'}$, we can see that this exactly equals $J(\pi') - J(\pi)$. Thus, for $f = V^{\pi'}$, we recover an exact equality. While this is not practically useful to us (because, when we want to optimize a lower bound with respect to $\pi'$, it is too expensive to evaluate $V^{\pi'}$ for each candidate to be practical), it provides insight: the penalty coefficient on the divergence captures information about the mismatch between $f$ and $V^{\pi'}$. 
\end{remark}

Next, we are interested in bounding the divergence term, $\|d^{\pi'} - d^{\pi}\|_1$. We give the following lemma; to the best of our knowledge, this is a new result. 

\begin{lemma}\label{divergencebound}
The divergence between discounted future state visitation distributions, $\|d^{\pi'} - d^{\pi}\|_1$, is bounded by an average divergence of the policies $\pi'$ and $\pi$:
\begin{equation}
\|d^{\pi'} - d^{\pi}\|_1 \leq \frac{2\gamma}{1-\gamma} \underset{s \sim d^{\pi}}{\E} \left[ D_{TV} (\pi' || \pi)[s]\right], 
\end{equation}
where $D_{TV} (\pi'||\pi)[s] = (1/2) \sum_a |\pi'(a|s) - \pi(a|s)|$.  
\end{lemma}

\begin{proof}
First, using (\ref{dpidiff}), we obtain
\begin{eqnarray*}
\|d^{\pi'} - d^{\pi}\|_1 &=& \gamma \|\bar{G} \Delta d^{\pi}\|_1 \\
&\leq & \gamma \|\bar{G}\|_1 \|\Delta d^{\pi}\|_1.
\end{eqnarray*}
$\|\bar{G}\|_1$ is bounded by:
%
\begin{equation*}
\|\bar{G}\|_1 = \|(I - \gamma P_{\pi'})^{-1}\|_1 \leq \sum_{t=0}^{\infty} \gamma^t \left\|P_{\pi'}\right\|_1^t = (1-\gamma)^{-1}
\end{equation*}

To conclude the lemma, we bound $\|\Delta d^{\pi}\|_1$. 
\begin{eqnarray*}
\|\Delta d^{\pi}\|_1 &=& \sum_{s'} \left| \sum_s \Delta(s'|s) d^{\pi}(s) \right| \\
&\leq& \sum_{s,s'} \left| \Delta(s'|s)\right| d^{\pi}(s) \\
&=& \sum_{s,s'} \left| \sum_a P(s'|s,a) \left(\pi'(a|s) - \pi(a|s) \right)\right| d^{\pi}(s) \\
&\leq& \sum_{s,a,s'} P(s'|s,a) \left|\pi'(a|s) - \pi(a|s) \right| d^{\pi}(s) \\
&=& \sum_{s,a} \left|\pi'(a|s) - \pi(a|s) \right| d^{\pi}(s) \\
&=& 2 \underset{s \sim d^{\pi}}{\E} \left[ D_{TV} (\pi'||\pi)[s] \right].
\end{eqnarray*} 
\end{proof}

The new policy improvement bound follows immediately. 

\performancebound*

%

\begin{proof} Begin with the bounds from lemma \ref{policybound0} and bound the divergence $D_{TV} (d^{\pi'} || d^{\pi})$ by lemma \ref{divergencebound}.
\end{proof}

\subsection{Proof of Analytical Solution to LQCLP}

\begin{theorem}[Optimizing Linear Objective with Linear and Quadratic Constraints] \label{thmlqclp}

Consider the problem
\begin{align}
p^* = \min_x &\; g^T x \nonumber \\
\text{s.t.} &\; b^T x + c \leq 0 \label{lqclp} \\
&\; x^T H x \leq \delta, \nonumber
\end{align}
where $g, b, x \in \Real{n}$, $c, \delta \in \Real{}$, $\delta > 0$, $H \in \Sym{n}$, and $H \succ 0$. When there is at least one strictly feasible point, the optimal point $x^*$ satisfies
\begin{align*}
x^* &= - \frac{1}{\lambda^*} H^{-1} \left(g + \nu^* b\right),
\end{align*}
where $\lambda^*$ and $\nu^*$ are defined by
\begin{align*}
\nu^* &= \left( \dfrac{\lambda^* c -  r}{s} \right)_+, \\
\lambda^* &= \arg \max_{\lambda \geq 0} \; \left\{ \begin{array}{ll}
f_a(\lambda) \doteq \frac{1}{2\lambda} \left(\frac{r^2}{s} -q\right) + \frac{\lambda}{2}\left(\frac{c^2}{s} - \delta\right) - \frac{rc}{s}  & \text{\textnormal{if} } \lambda c - r > 0 \\
f_b(\lambda) \doteq -\frac{1}{2} \left(\frac{q}{\lambda}  + \lambda \delta\right) & \text{\textnormal{otherwise}},
\end{array}\right.
\end{align*}
with $q = g^T H^{-1} g$, $r = g^T H^{-1} b$, and $s = b^T H^{-1} b$. 

Furthermore, let $\Lambda_a \doteq \{\lambda | \lambda c - r > 0, \lambda \geq 0\}$, and $\Lambda_b \doteq \{\lambda | \lambda c - r \leq 0, \lambda \geq 0\}$. The value of $\lambda^*$ satisfies
\begin{equation*}
\lambda^* \in \left\{ \lambda_a^* \doteq \text{\textnormal{Proj}}\left(\sqrt{\frac{q - r^2 /s}{\delta - c^2 /s}}, \Lambda_a\right), \lambda_b^* \doteq \text{\textnormal{Proj}}\left(\sqrt{\frac{q}{\delta}}, \Lambda_b \right)\right\},
\end{equation*}
with $\lambda^* = \lambda_a^*$ if $f_a (\lambda_a^*) > f_b (\lambda_b^*)$ and $\lambda^* = \lambda_b^*$ otherwise, and $\text{Proj}(a,S)$ is the projection of a point $x$ on to a set $S$. Note: the projection of a point $x \in \Real{}$ onto a convex segment of $\Real{}$, $[a,b]$, has value $\text{Proj}(x,[a,b]) = \max(a,\min(b,x))$.

\end{theorem}

\begin{proof}
This is a convex optimization problem. When there is at least one strictly feasible point, strong duality holds by Slater's theorem. We exploit strong duality to solve the problem analytically.   
\begin{align*}
p^* &= \min_x \underset{\begin{subarray}{c} \lambda \geq 0 \\ \nu \geq 0\end{subarray}}{\max} \; g^T x + \frac{\lambda}{2} \left( x^T H x - \delta \right) + \nu \left(b^Tx +c \right)\\
 &= \underset{\begin{subarray}{c} \lambda \geq 0 \\ \nu \geq 0\end{subarray}}{\max} \min_x  \; \frac{\lambda}{2} x^T H x + \left(g + \nu b\right)^T x + \left( \nu c - \frac{1}{2} \lambda \delta \right) && \text{Strong duality}\\
& \;\;\; \implies x^* = -\frac{1}{\lambda} H^{-1} \left(g + \nu b \right) && \nabla_x \calL(x,\lambda, \nu) = 0 \\
&= \underset{\begin{subarray}{c} \lambda \geq 0 \\ \nu \geq 0\end{subarray}}{\max}  \; -\frac{1}{2\lambda} \left(g + \nu b \right)^T H^{-1} \left(g + \nu b \right) + \left( \nu c - \frac{1}{2} \lambda \delta \right) && \text{Plug in } x^*\\
&= \underset{\begin{subarray}{c} \lambda \geq 0 \\ \nu \geq 0\end{subarray}}{\max}  \; -\frac{1}{2\lambda} \left(q + 2 \nu r + \nu^2 s\right) + \left( \nu c - \frac{1}{2} \lambda \delta \right) && \text{Notation: } q \doteq g^T H^{-1} g, \;\; r \doteq g^T H^{-1} b, \;\; s \doteq b^T H^{-1} b.\\
& \;\;\; \implies \diffp{\calL}{\nu} = -\frac{1}{2\lambda}\left( 2r + 2 \nu s \right) + c \\
& \;\;\; \implies \nu = \left(\frac{\lambda c - r}{s} \right)_+ && \text{Optimizing single-variable convex quadratic function over } \Realp{}\\
&= \max_{\lambda \geq 0} \;  \left\{ \begin{array}{ll}
\frac{1}{2\lambda} \left(\frac{r^2}{s} -q\right) + \frac{\lambda}{2}\left(\frac{c^2}{s} - \delta\right) - \frac{rc}{s}  & \text{if } \lambda \in \Lambda_a  \\
-\frac{1}{2} \left(\frac{q}{\lambda}  + \lambda \delta\right) & \text{if } \lambda \in \Lambda_b
\end{array}\right. && \text{Notation: } \begin{array}{ll}
\Lambda_a \doteq \{\lambda | \lambda c - r  > 0, \;\; \lambda \geq 0\}, \\ \Lambda_b \doteq \{\lambda | \lambda c - r \leq 0, \;\; \lambda \geq 0\}
\end{array} 
\end{align*}
Observe that when $c < 0$, $\Lambda_a = [0, r/c)$ and $\Lambda_b = [r/c, \infty)$; when $c > 0$, $\Lambda_a = [r/c, \infty)$ and $\Lambda_b = [0, r/c)$. 

Notes on interpreting the coefficients in the dual problem:
\begin{itemize}
\item We are guaranteed to have $r^2/s - q \leq 0$ by the Cauchy-Schwarz inequality. Recall that $q = g^T H^{-1} g$, $r = g^T H^{-1} b$, $s = b^T H^{-1} b$. The Cauchy-Scwarz inequality gives:
\begin{align*}
& \| H^{-1/2} b\|_2^2  \| H^{-1/2} g\|_2^2  \geq \left( \left( H^{-1/2} b\right)^T \left(H^{-1/2} g \right) \right)^2 \\
\implies & \left( b^T H^{-1} b \right) \left(g^T H^{-1} g \right) \geq \left(b^T H^{-1} g\right)^2 \\
\therefore \;\;\;& q s \geq r^2.
\end{align*}
\item The coefficient $c^2/s - \delta$ relates to whether or not the plane of the linear constraint intersects the quadratic trust region. An intersection occurs if there exists an $x$ such that $c + b^T x = 0$ with $x^T H x \leq \delta$. To check whether this is the case, we solve
\begin{equation}
x^* = \arg \min_x x^T H x \;\;\; : \;\;\; c + b^T x = 0
\end{equation}
and see if $x^{*T} H x^* \leq \delta$. The solution to this optimization problem is $x^* = c H^{-1} b / s$, thus $x^{*T} H x^* = c^2 /s$. If $c^2 /s - \delta \leq 0$, then the plane intersects the trust region; otherwise, it does not. 
\end{itemize}

If $c^2 /s - \delta > 0$ and $c < 0$, then the quadratic trust region lies entirely within the linear constraint-satisfying halfspace, and we can remove the linear constraint without changing the optimization problem. If $c^2 / s - \delta > 0$ and $c > 0$, the problem is infeasible (the intersection of the quadratic trust region and linear constraint-satisfying halfspace is empty). Otherwise, we follow the procedure below.

Solving the dual for $\lambda$: for any $A>0$, $B>0$, the problem
\begin{equation*}
\max_{\lambda \geq 0}  f(\lambda) \doteq -\frac{1}{2} \left(\frac{A}{\lambda} + B\lambda\right)
\end{equation*}
has optimal point $\lambda^* = \sqrt{A / B}$ and optimal value $f(\lambda^*) = -\sqrt{AB}$.

We can use this solution form to obtain the optimal point on each segment of the piecewise continuous dual function for $\lambda$:
\begin{align*}
& \text{objective} && \text{optimal point (before projection)} && \text{optimal point (after projection)}\\
\cline{1-6}
& f_a(\lambda) \doteq \frac{1}{2\lambda} \left(\frac{r^2}{s} -q\right) + \frac{\lambda}{2}\left(\frac{c^2}{s} - \delta\right) - \frac{rc}{s}  && \lambda_a \doteq \sqrt{\frac{q - r^2 /s}{\delta - c^2 /s}} && \lambda_a^* = \text{Proj}(\lambda_a, \Lambda_a)\\ 
& f_b(\lambda) \doteq -\frac{1}{2} \left(\frac{q}{\lambda}  +\lambda \delta\right) && \lambda_b \doteq \sqrt{\frac{q}{\delta}}  && \lambda_b^* = \text{Proj}(\lambda_b, \Lambda_b)
\end{align*}

The optimization is completed by comparing $f_a(\lambda_a^*)$ and $f_b(\lambda_b^*)$:
\begin{equation*}
\lambda^* = \left\{ \begin{array}{ll}
\lambda_a^* & f_a (\lambda_a^*) \geq f_b (\lambda_b^*) \\
\lambda_b^* & \text{otherwise}.
\end{array}\right.
\end{equation*}

\end{proof}

\subsection{Experimental Parameters}
\subsubsection{Environments} \label{envirosec}

In the Circle environments, the reward and cost functions are
\begin{align*}
R(s) &= \frac{v^T [-y, x]}{1+ \left| \|[x,y]\|_2 - d \right|}, \\
C(s) &= \pmb{1} \left[ |x| > x_{lim}\right],
\end{align*}
where $x,y$ are the coordinates in the plane, $v$ is the velocity, and $d, x_{lim}$ are environmental parameters. We set these parameters to be
\begin{center}
\begin{tabular}{c|ccc}
& Point-mass & Ant & Humanoid \\
\hline
$d$ & 15 & 10 & 10 \\
$x_{lim}$ & 2.5 & 3  & 2.5
\end{tabular}
\end{center}

In Point-Gather, the agent receives a reward of $+10$ for collecting an apple, and a cost of $1$ for collecting a bomb. Two apples and eight bombs spawn on the map at the start of each episode. In Ant-Gather, the reward and cost structure was the same, except that the agent also receives a reward of $-10$ for falling over (which results in the episode ending). Eight apples and eight bombs spawn on the map at the start of each episode.

\begin{figure}[t]
\centering
\begin{subfigure}{.22\textwidth}
  \centering
  \includegraphics[height=0.9\linewidth]{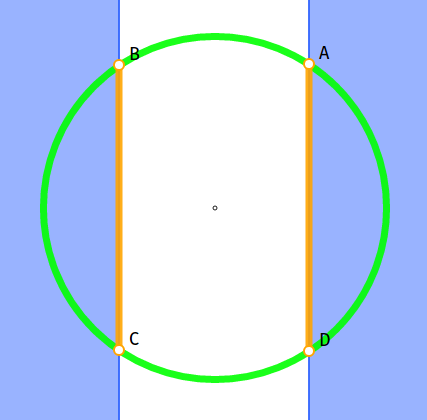}
\end{subfigure}
\caption[]{In the Circle task, reward is maximized by moving along the green circle. The agent is not allowed to enter the blue regions, so its optimal constrained path follows the line segments $AD$ and $BC$.  }
  \label{geom}
\end{figure}

\subsubsection{Algorithm Parameters}
In all experiments, we use Gaussian policies with mean vectors given as the outputs of neural networks, and with variances that are separate learnable parameters. The policy networks for all experiments have two hidden layers of sizes $(64,32)$ with $\tanh$ activation functions.

We use GAE-$\lambda$ \cite{Schulman2016} to estimate the advantages and constraint advantages, with neural network value functions. The value functions have the same architecture and activation functions as the policy networks. We found that having different $\lambda^{GAE}$ values for the regular advantages and the constraint advantages worked best. We denote the $\lambda^{GAE}$ used for the constraint advantages as $\lambda_C^{GAE}$. 

For the failure prediction networks $P_{\phi} (s \to U)$, we use neural networks with a single hidden layer of size $(32)$, with output of one sigmoid unit. At each iteration, the failure prediction network is updated by some number of gradient descent steps using the Adam update rule to minimize the prediction error. To reiterate, the failure prediction network is a model for the probability that the agent will, at some point in the next $T$ time steps, enter an unsafe state.  The cost bonus was weighted by a coefficient $\alpha$, which was $1$ in all experiments except for Ant-Gather, where it was $0.01$. Because of the short time horizon, no cost bonus was used for Point-Gather. 

For all experiments, we used a discount factor of $\gamma = 0.995$, a GAE-$\lambda$ for estimating the regular advantages of $\lambda^{GAE} = 0.95$, and a KL-divergence step size of $\delta_{KL} = 0.01$.  

Experiment-specific parameters are as follows:
\begin{center}
\begin{tabular}{c|ccccc}
Parameter & Point-Circle & Ant-Circle & Humanoid-Circle & Point-Gather & Ant-Gather\\
\hline
Batch size & 50,000 & 100,000 & 50,000  & 50,000 & 100,000\\
Rollout length & 50-65 & 500 & 1000 & 15  & 500 \\
Maximum constraint value $d$ & 5 & 10 & 10 & 0.1 & 0.2 \\
Failure prediction horizon $T$ & 5 & 20 & 20 & (N/A) & 20\\
Failure predictor SGD steps per itr & 25 & 25 & 25 & (N/A) & 10 \\
Predictor coeff $\alpha$ & 1 & 1 & 1 & (N/A) & 0.01 \\
$\lambda_C^{GAE}$ & 1 & 0.5 & 0.5 & 1 & 0.5
\end{tabular}
\end{center}
Note that these same parameters were used for all algorithms. 

We found that the Point environment was agnostic to $\lambda_C^{GAE}$, but for the higher-dimensional environments, it was necessary to set $\lambda_C^{GAE}$ to a value $<1$. Failing to discount the constraint advantages led to substantial overestimates of the constraint gradient magnitude, which led the algorithm to take unsafe steps. The choice $\lambda_C^{GAE} = 0.5$ was obtained by a hyperparameter search in $\{0.5,0.92,1\}$, but $0.92$ worked nearly as well. 

\subsubsection{Primal-Dual Optimization Implementation} \label{pdoimplement}

Our primal-dual implementation is intended to be as close as possible to our CPO implementation. The key difference is that the dual variables for the constraints are stateful, learnable parameters, unlike in CPO where they are solved from scratch at each update. 

The update equations for our PDO implementation are
\begin{align*}
\theta_{k+1} &= \theta_k + s^j  \sqrt{\frac{2\delta}{ (g - \nu_k b)^T H^{-1} (g - \nu_k b)}} H^{-1} \left(g - \nu_k b\right) \\
\nu_{k+1} &= \left( \nu_k + \alpha \left( J_C (\pi_k) - d \right) \right)_+,
\end{align*}
where $s^j$ is from the backtracking line search ($s \in (0,1)$ and $j \in \{0,1,...,J\}$, where $J$ is the backtrack budget; this is the same line search as is used in CPO and TRPO), and $\alpha$ is a learning rate for the dual parameters. $\alpha$ is an important hyperparameter of the algorithm: if it is set to be too small, the dual variable won't update quickly enough to meaningfully enforce the constraint; if it is too high, the algorithm will overcorrect in response to constraint violations and behave too conservatively. We experimented with a relaxed learning rate, $\alpha = 0.001$, and an aggressive learning rate, $\alpha = 0.01$. The aggressive learning rate performed better in our experiments, so all of our reported results are for $\alpha = 0.01$. 

Selecting the correct learning rate can be challenging; the need to do this is obviated by CPO.

\end{document}